\documentclass[sigconf]{acmart}
\renewcommand\footnotetextcopyrightpermission[1]{} % removes footnote with conference information in first column
\usepackage{subfigure} 
\usepackage{graphicx}
\usepackage{epsfig}
\usepackage{natbib,hyperref}
\usepackage{doi}
\usepackage{amsmath}
\AtBeginDocument{%
  \providecommand\BibTeX{{%
    \normalfont B\kern-0.5em{\scshape i\kern-0.25em b}\kern-0.8em\TeX}}}

\setcopyright{acmcopyright}
\copyrightyear{}
\acmYear{}
\acmDOI{}

\acmConference[]{}{}{}
\acmBooktitle{}
\acmPrice{}
\acmISBN{}

\begin{document}

\title{ArchNet: A Data Hiding Design for Distributed Machine Learning Systems}

\author{Kaiyan Chang}
\email{changkaiyan@foxmail.com}

\affiliation{%
  \institution{University of Electronic Science and Technology of China}
  \city{Cheng Du}
  \state{Si Chuan}
  \country{China}}

\author{Wei Jiang}
\authornote{Corresponding author}
\email{weijiang@uestc.edu.cn}
\affiliation{%
  \institution{University of Electronic Science and Technology of China}
  \city{Cheng Du}
  \state{Si Chuan}
  \country{China}
}

\author{Jinyu Zhan}
\affiliation{%
  \institution{University of Electronic Science and Technology of China}
  \city{Cheng Du}
  \state{Si Chuan}
  \country{China}
}

\author{Zicheng Gong}
\affiliation{%
  \institution{University of Electronic Science and Technology of China}
  \city{Cheng Du}
  \state{Si Chuan}
  \country{China}
}

\author{Weijia Pan}
\affiliation{%
  \institution{University of Electronic Science and Technology of China}
  \city{Cheng Du}
  \state{Si Chuan}
  \country{China}
}

\begin{abstract}
  Integrating idle embedded devices into cloud computing is a promising approach to support distributed 
  machine learning. In this paper, we approach to address the data hiding problem in such distributed 
  machine learning systems. For the purpose of the data encryption
  in the distributed machine learning systems, we propose the Tripartite Asymmetric Encryption theorem and give
  mathematical proof. Based on the theorem, we design a general image encryption scheme ArchNet.
  The scheme has been implemented on MNIST, Fashion-MNIST and Cifar-10 datasets to simulate real
  situation. We use different base models on the encrypted datasets and compare the results with the RC4
  algorithm and differential privacy policy. Experiment results evaluated the efficiency of the proposed
  design. Specifically, our design can improve the accuracy on MNIST up to 97.26\% compared with RC4.
  The accuracies on the datasets encrypted by ArchNet are 97.26\%, 84.15\% and 79.80\%, and they are
  97.31\%, 82.31\% and 80.22\% on the original datasets, which shows that the encrypted accuracy of ArchNet
  has the same performance as the base model. It also shows that ArchNet can 
  be deployed on the distributed system with embedded devices.
\end{abstract}


\keywords{Distributed machine learning  Embedded systems  Data Encryption  Neural Networks}


\maketitle

\section{INTRODUCTION}
Cloud computing services have become the de facto standard 
technique for training neural network. Some machine learning service providers (\emph{e.g.}, Google) 
develop TPU or FPGA to accelerate the neural network calculations of cloud servers  \cite{fpga,tpu,JIANG2020101775}. 
They encourage people to upload datasets to the cloud. However, there are two problems with this business 
model. (i) Although it is convenient for users to use cloud machine learning services, each upload can only 
use algorithms implemented by the service provider. (ii) The resources on cloud servers are limited by hardware, 
which can not change with heavy load. 
% In order to address the first problem, many Machine Learning (ML) server 
% providers allow users to rent cloud servers to implement their own algorithms. Some companies put more hardwares in 
% their data centers to solve the second problem.
Despite the first solution, some users do not have the corresponding expertise to design better models for 
self-training. They are not satisfied with the fixed performance of commercial models. They want to try newer 
and more diverse algorithms to enhance the competitiveness of users' own services. For the second solution, 
compute resources lack flexibility. Therefore, we proposed a distributed machine learning system with 
embedded devices. The cloud service provider only serves as an intermediary between the data publisher and 
the algorithm owner on embedded device. 

Distributed machine learning system has severe security issues. It may leak sensitive personal 
information(\emph{e.g.} bank card password, identity number). Patented algorithms may also be leaked
in the process. Untrusted nodes in the system can steal both data and model 
through the exchanged information. Some recent studies show that the gradients sharing in such system
 may leak important information\cite{deepleak}.
 The Fully Homomorphic Encryption (FHE) algorithm\cite{ChouFaster} proposed in recent years can encrypt data.
 However, the encrypted data rely on a special neural network structure to be correctly identified. 
 Existing methods do not combine the accuracy and the data protection well, and embedded devices are 
 not well integrated with cloud servers.

In this paper, we first introduce a novel and effective distributed machine learning system with embedded 
devices (\emph{i.e.}, moving the computing end of machine learning from the central server to the remote end 
on embedded devices). We further propose a dataset encryption scheme (ArchNet), which can solve the problem 
of untrustworthy in embedded devices of the distributed machine learning computing system. We deduce the 
basic principle of Tripartite Asymmetric Encryption mathematically. We prove that neural network can be 
used for encryption and decryption. We find that using neural network as encryptor can make the encrypted 
dataset difficult to be stolen by others and easy to learn at the remote end. The reasons are that some 
basic unit combinations of neural network have reversible operations, and it is difficult for human to 
recognize data in high-dimensional space.
Our primary contributions in this paper are as follows.
\begin{itemize}
  \item We identify the data hiding problem in distributed machine learning systems with embedded devices.
  \item We propose Tripartite Asymmetric Encryption and two kinds of key, which can be used on embedded devices.
  \item We prove the rationality of neural network in data encryption, which enhances the theorem base of the encrypted neural network.
  \item We design and implement the data encryption scheme ArchNet to address the data hiding problems in distributed machine learning systems, which significantly outperforms than the existing approach and has the same performance as base model.
\end{itemize}

The rest of the paper is organized as follows. Section 2 provides the principle of distributed machine learning on embedded device. Section 3 summarizes the necessary theoretical basis of data hiding, and describes the basic principle of Tripartite Asymmetric Encryption. Section 4 provides the ArchNet scheme to implement a data encryption scheme. We present our experimental results in Section 5. Section 6 summarizes the related work and Section 7 concludes the paper.
\section{DISTRIBUTED MACHINE LEARNING PRELIMINARIES}\label{hhhh}
In this section, we first describe the structure of distributed machine learning system. Finally, we demonstrate why we have to solve the model hiding problem and data hiding problem in the system.
\subsection{Distributed Machine Learning System with embedded Devices}
The computing resources of the cloud servers are limited by hardware and the fixed algorithms of service provider. 
In order to solve this problem, we propose distributed machine learning system which is different from the distributed 
machine learning on heterogeneous computing system (\emph{i.e}, a computer with GPU, TPU or FPGA). It is a change of the business pattern in cloud ML. Computer network is the basis of our system.
\begin{figure}[!h]
  \centering
  \includegraphics[width=\linewidth]{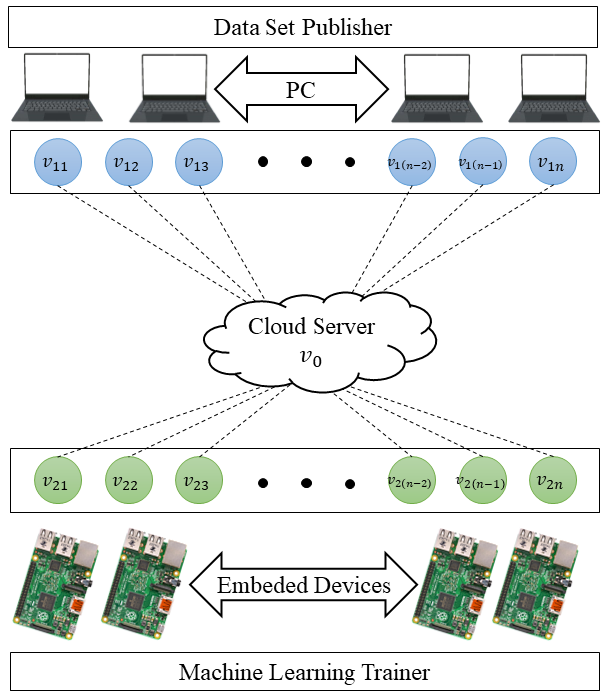}
  \caption{Distributed Machine Learning System Structure}
  \label{distrifig}
\end{figure}
We define the computer network structure as $G=\left \langle V_c,E_c \right \rangle $. Here $V_c$ denotes the union of 3 different sets as shown in Fig. \ref{distrifig} machine learning server node $v_0$, dataset publishing node sets$\left\{v_{11},v_{12},\cdots,v_{1n}\right\}$, and computer node (\emph{i.e.} embedded device) sets with computing resources $\left\{v_{21},v_{22},\cdots,v_{2m}\right\}$. The machine learning server node ${v_0}$  receives datasets from $v_{1i},1 \le i \le n$ via the internet, and then post the task information. 
The computer that meets the task requirements $v_{2j},1\le j\le m$ can make a request to the machine learning server $v_0$ automatically, which send the dataset to the computer $v_{2j}$ through the internet. $v_{2j}$ provides deep learning algorithm to train the dataset. When the training is finished, the embedded device sends the deep learning model back to the machine learning server, and the server pays a certain amount of fees to the computer after using the validation dataset. Finally, the server sends the deep learning model back to the dataset publisher node.

In this system, machine learning server does not need to have computing resources. It behaves as an intermediary of data and algorithm deployment. $v_{1i}$ is the publisher of dataset and the user of trained deep learning model. $v_{2j}$ is both a publisher of the deep learning model and a user of the dataset. (\emph{e.g.}, $v_{2j}$ can be a university or a scientific research institute, which wants to use the dataset provided by $v_{1i}$ to validate their new deep learning algorithms.)

Compared with the cloud ML system, our system has no resource restrictions on the server. It can help the dataset publishers find specified algorithms. The system makes algorithm designers more active in validating their algorithms.
\subsection{Model Hiding}
For computers $v_{2j},1\le j\le m$, their users don't want others to know their private training algorithms. $v_{2j}$ serializes the 
model into a universal format (\emph{e.g.}, ONNX) which connects the deep learning model publisher and the dataset publisher. Because 
the model in universal format does not contain the optimization method and data augmentation algorithm, it can hide the algorithm of the 
deep learning model. The $v_{2j}$ node knows how to train the model, while the $v_{1i}$ node does not know the specific training algorithm. (\emph{e.g.}, suppose a scientific research group decides to test their new algorithm for image classification on the computer $v_{21}$, they submit a dataset request to the machine learning service provider $v_0$, then use their new algorithm training the dataset and generate ONNX file. Ultimately, they send ONNX file back to the server $v_0$. The server validates the deep learning model after they complete the tasks. The server do not know the specific algorithm of training.) The private training algorithm can be protected well in our system. Only $v_{2j}$ knows the crucial training method.
\subsection{Data Hiding}
For publishers with datasets $v_{1i},1\le i\le n$, their users do not want others to know the contents of their datasets for the purpose of protecting privacy and sensitive information. (\emph{e.g.}, the dataset publisher $v_{1i}$ publishes the dataset $S$. Although $v_0$ and $v_{2j}$ can get the dataset $S$, they cannot know the meaning represented by $S$. They can use the dataset $S$ to train the deep learning model.) In order to hide data, we design an encryption algorithm. This general encryption algorithm like AES and RC4 has been applied to data on the internet. However, the use of such encryption algorithm for data hiding will cause neural network training problems. This algorithm disrupts the original distribution of the dataset, resulting in the accuracy of neural network is very low. Therefore, some researchers propose FHE. However, the current FHE does not have a universality in deep learning algorithms. Only specific models can be used as base model (\emph{e.g.}, CryptoNet). For the rest of the paper, we focus on the training accuracy of data encryption and propose our encrypt method (ArchNet) to improve the accuracy of neural network on encrypted dataset.
\section{PROBLEM FORMULATION AND THE SOLUTION}
In this section, we give the mathematical proof of the tripartite asymmetric encryption. We finally propose standards to evaluate the performance of an encryption scheme.
\subsection{Data Hiding Problem}
An optimization problem usually consists of three different components: a vector of parameters $x$, an objective function $F(x)$, and a set of constraint functions $C_i(x)$. The goal is to find a concrete value of the parameter vector $x$ that maximize $F(x)$ while satisfying all constraint functions $C_i (x)$ as shown below.
\begin{displaymath}
  \min F(x)
\end{displaymath}
\begin{equation}
  \begin{split}
    \mbox{s.t.}\quad
    C_i(x)\le 0, i\in N  \\
    C_i(x)=0, i\in Q
  \end{split}
\end{equation}
Here $x\in R^n,R,N,Q$ denote the sets of real numbers, the  indices for inequality constraints, and the indices for equality constraints, respectively.

Usually, we use iterative method with a small learning rate to maximize $F(x)$. We must make sure the result meet the constraints after each iteration. However, most constraint functions can not be expressed analytically in practical applications. Evolutionary algorithm is generally used to solve the optimization problem with no gradient information, while back propagation in neural network is used to solve the optimization with obvious gradient information. Both of them can only solve unconstrained optimization problems. It is difficult to simplify constrained optimization problem to unconstrained optimization problem based on specific problems. Data hiding is a constrained optimization problem. The goal of the problem is to maximize the distance between the original data distribution function $f(x)$ and the encrypted data distribution function $g(x)$. The constraint of the problem is that the encrypted data $g(x)$ can still be recognized with high accuracy by deep learning algorithms. The constraint is the premise of ensuring the flexibility of algorithm in distributed machine learning system. The rest of the paper focuses on the simplification of data hiding problem so that it can be solved by neural network.

\subsection{Tripartite Asymmetric Encryption}
In order to transform the data hiding problem into unconstrained optimization problem, we eliminate the constraint function $C_i(x)$. Assume the deep learner as a set of functions. The first problem to solve is what kind of function $h$ can act on the dataset $S$ to make the deep learner still get correct classification for $h(S)$. The function $h$ with this property can be used as encryption function. The following definitions define the above problems mathematically.
\begin{definition}
  Given the countable infinite function set as:
  \begin{displaymath}
    F=\left\{f_0,f_1, \cdots \right\}, f_i\circ f_j=I
  \end{displaymath}
  where $I$ denotes the unit mapping, $f_i$ is called the first type of encoding function or the first type of encryption function (E1) corresponding to $f_j$, $f_j$ is called the first type of decoding function or the first type of decryption function (D1) corresponding to $f_i$.
\end{definition}
\begin{definition}
  Given the countable infinite function set as:
  \begin{displaymath}
    F=\left \{f_0,f_1,\cdots \right \}, f_i\circ f_j=g
  \end{displaymath}
  here $f_i$ is the second type of encoding function or the second type of encryption function (E2) corresponding to $f_j$, and $f_j$ is the second type of decoding function or the second type of decryption function (D2) corresponding to $f_i$, where the definition domain of $g$ is dataset $S$, the range of $g$ is the label set corresponding to dataset $S$.
\end{definition}
The definition above further describes that the first type of decoding function can decrypt the encrypted data to the original data, while the second decoding function can decrypt the encrypted data to the corresponding label. We intend to find a function with such a special property, which is both the first type of decoding function and the second type of decoding function.
\begin{theorem}\label{theorem1}
  If the function $f$ is the first type of encoding function, and $f$ is also the second type of encoding function, then it has the first type of decoding function $g_1$ and the second type of decoding function $g_2$.
\end{theorem}
\begin{proof}
  First, it has $f\in E1,E2$. Hence, it exists $f\circ g_1=I,f_i\circ g_2=g $. Therefore, $g_1\in D1,g_2\in D2$.
\end{proof}
Under theorem \ref{theorem1}, for pattern recognition problem, there is now 3 propositions. Proposition P: Deep learning model is used as decoder. Proposition Q: The target classification label is used as the supervision output. Proposition R: Original dataset is used as supervised output. It is easy to see that objects with $P\wedge Q=1$ need to satisfy functions $f$ and $g_1$, and objects with $P\wedge R=1$ need to satisfy functions $g_2$. (\emph{e.g.}, in our distributed machine learning system, in the view of computers with computing resources, it can be classified only when it satisfies the conjunction expression $P\wedge Q=1$, and in the view of dataset publisher, it can be encrypted only when it satisfies the conjunction expression $P\wedge R=1$.)

The above mathematical derivation provides a mathematical proof for our distributed machine learning system. The publisher of the dataset needs to have the encryption function $f$ and the first type of decoding function $g_1$. The computer with computing resources need to have the second type of decoding function $g_2$. Under our theorem, pattern recognition can continue under the data is encrypted. Theorem \ref{theorem1} is an extension of asymmetric encryption and the basic form of Tripartite Asymmetric Encryption (TAE). The following shows how TAE in distributed machine learning system can be implemented by eliminating constraints.

We use the analysis method to find the key of the problem. The constraint of the data hiding problem is that the second kind of decryptor must identify the encrypted data with high accuracy. In order to eliminate the constraint, we assume the second type of decryptor is an unconstrained optimization model, and this unconstrained optimization model as the second type of decryptor can be able to better recognize the encrypted dataset in all distribution.

We assume that the second type of decryptor is a neural network model. Neural networks can theoretically fit arbitrary nonlinear functions. In order to prove that the neural network model can recognize the encrypted dataset better, we define the concept of Function Compound Closure(FCC).
\begin{definition}\label{defination3}
  Given countable functions set $G_0=\left\{g_0,g_1,\cdots \right\}$. Then can construct a new set $H$ as follows:

  1. For all $g_i\in G_0 (i\in \left\{0,1,\cdots \right\}),g_i\in H$.

  2. For any $h_i,h_j\in H$, if $h_i,h_j$ can be compound, and there exists inverse operation of compound operation, such that $h_i\circ h_j\in H$, $(h_i\circ h_j )^{-1}\in H$.

  Then set $H$ is called the closure of set $G_0$ under the function composition(FCC).
\end{definition}
Suppose the neural network contains convolution, full connection, pooling and activation function operation structure \cite{dl}, then the neural network is a function composition. However, we expect the neural network model is a closure under the function composition (FCC), so we can explain that there exists a neural network that can decode a dataset which is mapped to a higher dimension by neural network according to the fact that there exists an inverse operation of compound operation in the definition \ref{defination3}. Its mathematical expression is as follows.
\begin{theorem}\label{theorem2}
  Suppose the dataset $S$ can be classified correctly, and there exists a neural network $D$ that can decode the data encoded by neural network $E$, then there exists a neural network mapping the $S'$ back to the low-dimensional space $S$, where $E$ maps the dataset $S$ to the high-dimensional space $S'$.
\end{theorem}
\begin{proof}
  Let Neural Network FCC
  $$G_0=\left\{conv_0,linear_0,pool_0,conv_1,linear_1,relu\cdots \right\}$$

  Let a dataset $$S=\left\{X_0,X_1,\cdots ,X_n \right\}$$

  here $X_j,0\le j\le n$ has $i$ attributes. 
  It exists a function compound operation $E$ from $G_0$, which makes dataset $S$ be mapped to high-dimensional space into dataset $S'$. Hence, $G_0$  is FCC. It exists $k\in G_0$ can reduce the data dimension(\emph{e.g.}, fully-connection).  It will exist neural network $N_0$, which can reduce $S'$ to the original low dimension space and restore it to the dataset $S$. Therefore, $S'$ holds all the feature information about $S$. For $S$ can be classified correctly, 
  it will have neural network $N$ which can classify $S'$.

\end{proof}
Theorem \ref{theorem2} demonstrates that there exists a neural network that can decode the encrypted data to its original state or its correct classification, where the encrypted data is encoded by a neural network encoder that maps the data to a higher dimension. Therefore, we provide a complete mathematical proof for the use of neural network as an unconstrained optimization tool to solve data hiding problems, which is also the theoretical basis of our encrypt algorithm  implementation(ArchNet). We propose measurement standards to evaluate the performance of an encryption model. We expect the encryption method is difficult to be cracked by malicious users, and it is easy for a deep learning model to recognize the encrypted data pattern.
\subsection{Difficulty to Steal}
For data hiding, we expect the owner of the encrypted dataset can not obtain the original dataset without the first type of decryption function. We have proved that in order to obtain the first type of encryption function, the proposition R in section 3.2 must be satisfied. But for the owner of encrypted dataset, proposition R is not satisfied. It can not steal the original dataset theoretically. Therefore, the dataset encrypted by our policy is difficult to steal.
\subsection{Computability}
The data encryption method should make computer recognize patterns better through the deep learning model. The encryption algorithm with high computability should ensure the accuracy of dataset $P_0$ and dataset $P_1$ is approximately equal on the validation set when the number of training epoch is same.
\begin{equation*}
  Acc_{epoch}(P_0,f)\approx Acc_{epoch}(P_1,f)
\end{equation*}

here $P_0$ denotes the original dataset and $P_1$ denotes the encrypted dataset. 
Choosing the validation accuracy to measure the computability of encryption method largely depends 
on the quality of the base model. In order to remove the influence of the base model, we propose an 
indicator $EC$. Suppose we solve a image classification problem $P_0$, $E_0$ encryption method acts 
on the original dataset. The computability of encryptor $E_0$ with classifier $f_0$ is expressed as Equ. \ref{equa33}.
\begin{equation}\label{equa33}
  EC_e(P_0,E_0,f_0)=\frac{Acc_e (P_0,f_0)-Acc_e(E_0(P_0),f_0 )}{Acc_e(P_0,f_0)}
\end{equation}

here $EC$ denotes easy to calculation, and $e$ denotes the number of training epoch. During the same epochs, the higher the $EC$ is, the worse the encryption method computability is. The encrypted data is difficult to recognize for deep learning model. The lower the $EC$ is, the better the encryption method is. The encrypted data is easy to recognize for deep learning model.
\section{IMPLEMENTATION OF ARCHNET}
Fig. \ref{highlevel} presents a high level overview of our approach. We describe the key components and the motivating examples of ArchNet in detail below.
\subsection{Overview of ArchNet}
\begin{figure}[h]
  \centering
  \includegraphics[width=\linewidth]{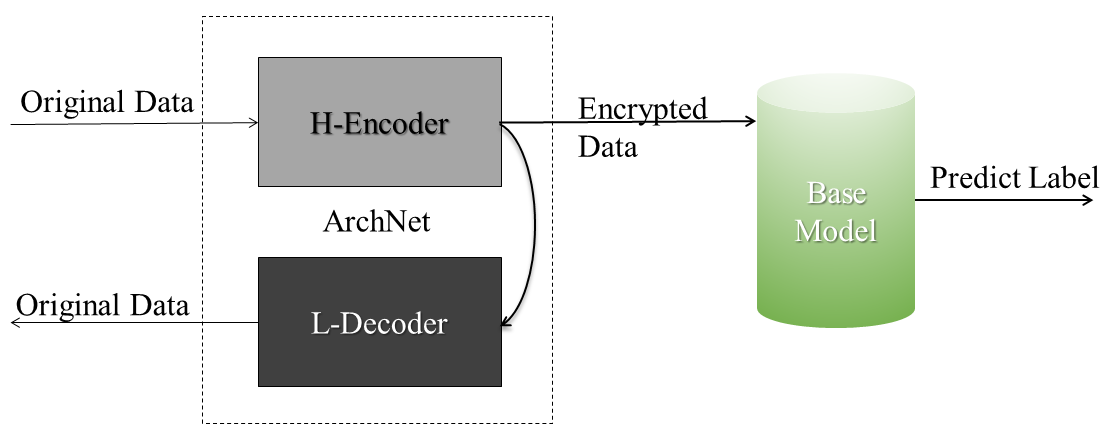}
  \caption{The high level overview of the ArchNet.}
  \label{highlevel}
\end{figure}
According to Theorem \ref{theorem2}, the premise of using neural network as decoder is that it can map data to high-dimensional space. Image is a tensor and it has three dimensions: channel, height and width. The channel is displayed in pixel color, and the length and width are displayed in pixel location. The target of encryption method is to prevent people from recognizing the encrypted data. We propose a method to map the original data to high-dimensional space. It is different from the denoising self-encoder \cite{noising} to reduce the data dimension. People can recognize data less than 3 dimensions, but can not recognize data over 3 dimensions.
According to Theorem \ref{theorem1}, in the process of training encryptor and decryptor, we keep the combination of them to produce unit mapping as shown in Fig. \ref{highlevel}. Therefore, the input dataset and the target dataset of the training model are the same. When the training is over, the input dataset and the target output dataset are split from the middle high-dimensional data. The first model is the first type of encoding function, and the second mode is the first type of decoding function. The dataset publisher can encrypt the training data by the first model to obtain the encrypted dataset. The dataset user does not need to have any part of the first type of functions. The image shape of the model layered output is shown in Fig. \ref{architec}, which is similar to an arch. We call this encryption and decryption method ArchNet. The first kind of encoding function which maps the original dataset into high-dimensional space is called H-encoder. The first kind of decoding function which restores the encrypted dataset from high-dimensional space to the original dataset is called L-decoder.
\begin{figure}[h]
  \centering
  \includegraphics[width=\linewidth]{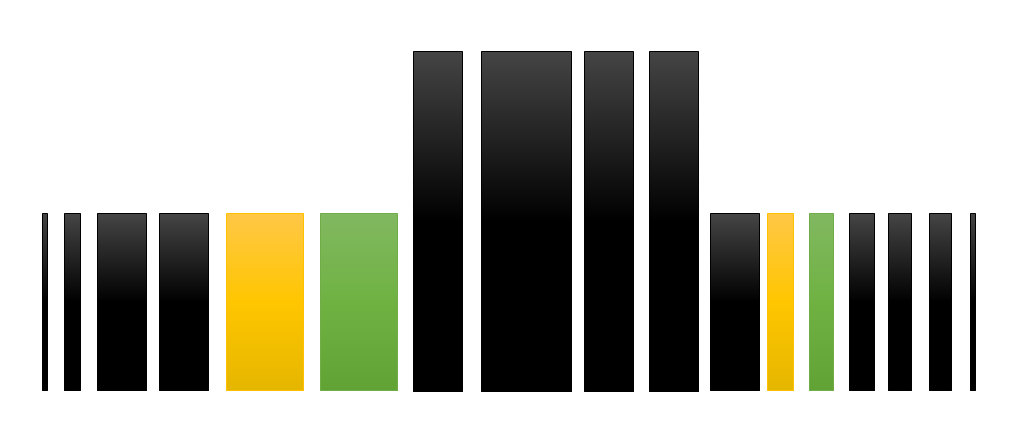}
  \caption{The concrete architecture of the ArchNet. Yellow presents the fully-connected layer. Blue presents the Relu activation function. Black presents the convolutionion layer.}
  \label{architec}
\end{figure}
\subsection{H-encoder}
H-encoder consists of several basic modules of neural network. It can not be training separately without L-decoder. In order to make H-encoder keep the original distribution of data, we primarily focus on convolution layer. In order to make the model more difficult to steal when encoder is applied to simple dataset, fully-connected layer and activation function (\emph{i.e.}, ReLU, Softmax, Tanh etc.) need to be added to H-encoder. Because pooling layer will lose data information, we do not recommend adding pooling layer to it. It is not suitable to use pure convolution layer, because convolution layer is more regular and can not hide the data. The data in the middle layer can also be recognized by human beings if only use convolution layer as shown in Fig. \ref{pureconv}. The output of H-encoder is high-dimensional data. In order to expand the dimension of data, we add transpose convolution layer at the end of H-encoder.
\subsection{L-decoder}
L-decoder is the implementation of first type of decoding function. 
The high-dimensional output of H-encoder is the input of L-decoder, whose goal is to remap the high-dimensional output to the original dataset. Convolution layers are included in L-decoder. Convolution layer can retain the local characteristics of data. For simple datasets, L-decoder includes fully-connected layer and activation function. In principle, 
L-decoder and H-encoder are symmetrical in structure. The purpose of unit mapping can be achieved 
by combining both of them. The difference between the two is that L-decoder does not have transposed convolution layer. 
We demonstrate of neural network design schemes from the dataset structure. For example: For the simple dataset (\emph{i.e.}, MNIST), 
fully- connected layer and activation function are added in the neural network to increase the complexity of the encrypted dataset. For the complex dataset (\emph{i.e.}, Cifar-10), only convolution layer is used in the neural network to increase the  computability of encrypted dataset.
\subsection{Training Strategy}
ArchNet can be used to encrypt a variety of data under different tasks. This paper focuses on data encryption in image classification. 
The quality of the base model is significant when training the encrypted dataset. Considering the learning strategy of ArchNet, suppose 
there is a training set $(X,Y)$, $X$ is an image set, $y$ is a label set corresponding to the image set. The training task of ArchNet 
includes parameter function $f(x,\theta )=x$. The goal is to obtain the parameter $\widehat{\theta}  $ such that $$\widehat{\theta }=argmin_{\theta } \sum_{x\in X}L(x,f(x,\theta ))$$where $L(x,f(x,\theta ))$ defines the loss between the output of ArchNet and the real output $X$. We use the binary cross entropy as the loss function, then the loss function is defined as:
\begin{equation}
  -\frac{1}{n}\sum_{i=1}^n [x\cdot \log \left(f(x,\theta )\right)+(1-x)\cdot \log \left(1-f(x,\theta )\right)]
\end{equation}
Here $n$ denotes the number of elements in the dataset. ArchNet achieve better results by using uniform distribution to initialize parameters. The neural network with gradient back propagation outperforms other evolutionary optimization strategies in training convergence time.
\subsection{Universality}
ArchNet is the implementation of TAE, which is an encryption algorithm in deep learning system. ArchNet can be expressed as $f(g(x))=x$, where $f$ is the first decoding function and $g$ is the first encoding function. In the general strategy, $h(g(x))=l$, where $l$ denotes the label of $x$ and $h$ denotes the second type of decoding function. ArchNet supports a diversity of base deep learning algorithms in distributed machine learning system. It is different from FHE, there are many choices of function $h$. Compared with the traditional general encryption algorithms such as RC4 \cite{rc4}, the ArchNet scheme is more stable. It adapts to a variety of base models while maintaining a higher accuracy. The accuracy difference between the ArchNet scheme and the original model is less than 1\%, which makes the $h(g(x))=l$ equation more consistent.
\subsection{Data Preprocessing}
Data quantity, data quality and data distribution as three representations of data affect the effect of distributed machine learning from different perspectives. In distributed machine learning system, the amount of data affects the efficiency of internet transmission. If the data quality is miserable, the accuracy of the base model is relatively low, and the accuracy of the encrypted data training is not high. The impact of data quality on ArchNet is significant. As the first type of encryption and decryption function, the encrypted data still has similar structural characteristics with the original data. Therefore, the data quality determines the data distribution. The data distribution affects the encrypted data distribution after ArchNet encryption. To get a better pattern recognition accuracy in distributed machine learning, we need to remove the noise in the data as much as possible to ensure that the original data can perform well in the base model.
\subsection{Time Delay}
In the distributed machine learning system, time delay is inevitable, especially on embedded devices. If time delay is defined as $t_0$, then $t_0=t_1+4t_2+t_3+t_4$, where $t_1$ denotes the delay of training ArchNet scheme, $t_2$ denotes the delay of network data transmission, 
which is divided into four parts: data publisher to machine learning server, machine learning server to data user (\emph{i.e.}, computers with computing source) to machine learning server, machine learning server to data publisher. Since the four parts are all network delays, we approximately equivalent them to the equivalent delay $t_2$. $t_3$ denotes the time of computing on the embedded devices, and $t_4$ denotes the time of data staying in the machine learning server and waiting for allocation. From the actual situation, the delay of $t_2$  is the smallest, $t_3$ is different because of the difference in training devices. In general, $t_3$ delay is the largest. The generation of $t_1$ is determined by the size of dataset and its batch size. The performace of embedded devices are not the first aspect of the time delay.
\subsection{An illustrative Example}
\begin{figure}[h]
  \centering
  \includegraphics[width=\linewidth]{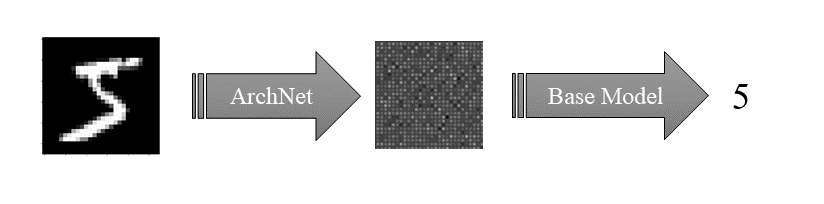}
  \caption{The processing of the MNIST dataset.}
  \label{process}
\end{figure}
As shown in Fig. \ref{process}, we take MNIST dataset as an example to demonstrate the data flow in distributed machine learning system. MNIST dataset is a typical image classification dataset. Now suppose the dataset publisher wants to find a model to recognize handwritten numbers, but it is not like using its own computer. It is not sure that its algorithm is better. So it uses MNIST dataset to train ArchNet for encryption, then encrypts the MNIST dataset and sends it to machine learning service provider. Machine learning service provider receives the encrypted dataset and looks for a device that wants to receive computing resources. The device receives the encrypted MNIST dataset and begins to use its own design of deep learning algorithm training model. After the training, the trained model will be sent back to the service provider, and the service provider will validate the accuracy and then pass it to the dataset publisher.
\section{EXPERIMENTAL EVALUATION}
In this section, we discuss our implementation and how we fine-tune ArchNet to achieve optimal performance.

All our measurements are performed on a system running Ubuntu 16.9 with NVIDIA GTX 2080 Ti GPU. Dataset generation is implemented on Intel (R) core (TM) i7-8700 CPU @ 3.20GHz.

\subsection{Feasibility to Embedded System}
As shown in Fig. \ref{framework},  our system can be divided into two stages. (i) Data set encryption on PC. (ii) Pattern recognition on embedded devices. Therefore, we can divide it into two training processes on one computer as a simulation and measure their effects. We encrypt data set with ArchNet on the first stage to simulate the process on PC and measure the performance of base model to simulate embedded devices.

\begin{figure}[h]
  \centering
  \includegraphics[width=\linewidth]{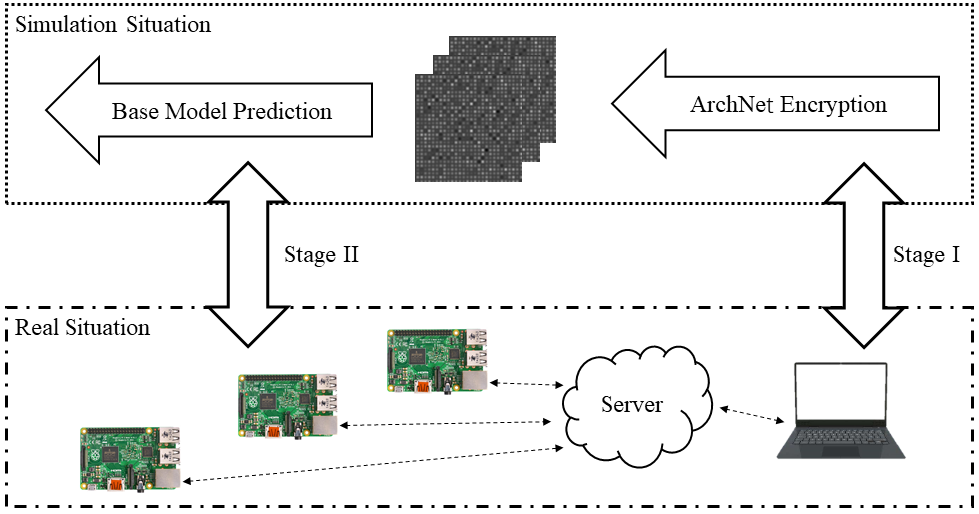}
  \caption{The framework of our evaluation.}
  \label{framework}
\end{figure}

\subsection{Model Architecture}
Our ArchNet scheme is implemented on MNIST, Fashion-MNIST and Cifar-10 datasets with pytorch-1.3.0. H-encoder consists of four convolution layers, one transposed convolution layer and one fully-connected layer. The fully-connected layer uses ReLU as its activation function. This function intends to increase the nonlinearity of the encoder. The tensor size is doubled in the transposed convolution layer. L-decoder consists of 8 convolution layers, two of which use ReLU as activation function. The ArchNet scheme was trained 10 epochs (\emph{i.e.}, 10 complete passes of the dataset) to achieve a high accuracy (\emph{e.g.}, the average accuracy in MNIST dataset is above 99.9\%). The training time for the MNIST dataset using ArchNet is less than 10 minutes, and that of Cifar-10 dataset is less than 15 minutes. The layer with the largest parameter quantity of ArchNet scheme is fully-connected layer. The number of parameter in ArchNet is $122,949,880$. Because Cifar-10 is a complex dataset, it does not need a full connection layer to increase the difficulty of stealing. Table 1 shows the ArchNet architecture in detail.

% \usepackage{booktabs}
\begin{table*}[]
  \caption{Evaluation on ArchNet and RC4}
  \begin{center}
    \begin{tabular}{c|ccc|ccc}
      \hline
      Method   & \multicolumn{3}{c}{ArchNet} & \multicolumn{3}{c}{RC4}                                         \\ \hline
      Dataset  & AO                          & AE                      & EC      & AO      & AE      & EC      \\ \hline
      F-MNIST  & 82.31\%                     & 84.15\%                 & -2.23\% & 82.31\% & 10.60\% & 87.49\% \\
      MNIST    & 97.31\%                     & 97.26\%                 & 0.05\%  & 97.31\% & 12.65\% & 87.00\% \\
      Cifar-10 & 80.22\%                     & 79.80\%                 & 0.52\%  & 80.22\% & 10.65\% & 86.72\% \\ \hline
    \end{tabular}
  \end{center}
  \label{indextable}
\end{table*}

\subsection{Training Data}
For each validation dataset, we construct H-encoder and L-decoder and train them for 50 epochs. As Fig. \ref{performance}, the results shows that their accuracy are more than 99\%. For MNIST and F-MNIST datasets, the training samples to validation samples ratio is 6:1. For Cifar-10 datasets, the training samples to validation samples ratio is 5:1. We get H-encoder at the end of ArchNet. And use H-encoder to encrypt dataset to get encrypted dataset. We use encrypted dataset to train the selected base model to train the final model. The accuracy of the final training model is verified by encryption validation set, and the result is obtained by $EC$ value.
\begin{figure}
  \includegraphics[width=\linewidth]{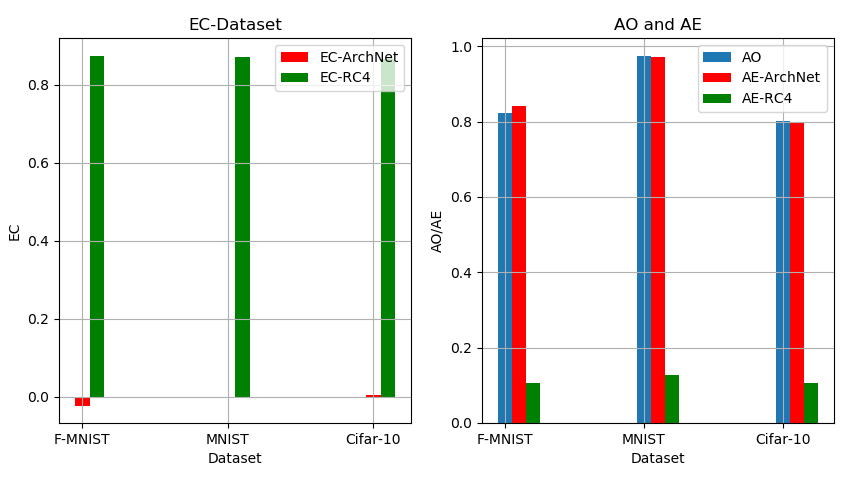}
  \caption{$EC$, $AE$ values of RC4 and ArchNet compared with $AO$ on F-MNIST, MNIST, Cifar-10.}
  \label{performance}
\end{figure}
\subsection{Training Strategy}
For ArchNet, we use Adam as the optimization algorithm. The initial learning rate is $10^{-5}$. 
We choose the mean square error function as the loss function which can directly show the difference
 between the output image and the real image. In the initialization of neural network parameters, 
 we use the uniform initialization method. We use Resnet151 \cite{resnet} as the base model of 
 Cifar-10 and convolutional neural network as the base model of MNIST and F-MNIST. The two basic 
 models can better evaluate the effect of ArchNet. The base models are trained on MNIST dataset, 
 Cifar-10 and F-MNIST dataset for 100 epochs. For the comparative experiment DP, we use tensorflow-privacy package with SGD optimal strategy.
\subsection{Evaluation Results}
The accuracy rate and relevant information of $EC$ of validation dataset used in our experiment
 are shown in table \ref{indextable}. Where, $AO$ denotes the accuracy of the original dataset, 
 $AE$ denotes the accuracy of the encrypted dataset. $EC$ is defined by formula \ref{equa33}.
$AO$ value is related to base model and dataset. $AE$ value is primarily related to base model, 
but also related to encryption policy. $EC$ value is primarily related to encryption policy. Our 
encrypt method hardly affects the accuracy of pattern recognition as shown in Fig. \ref{accuracy1}.
\begin{figure*}[htbp]
  \centering
  \subfigure[pic1.]{
    \includegraphics[width=5.5cm]{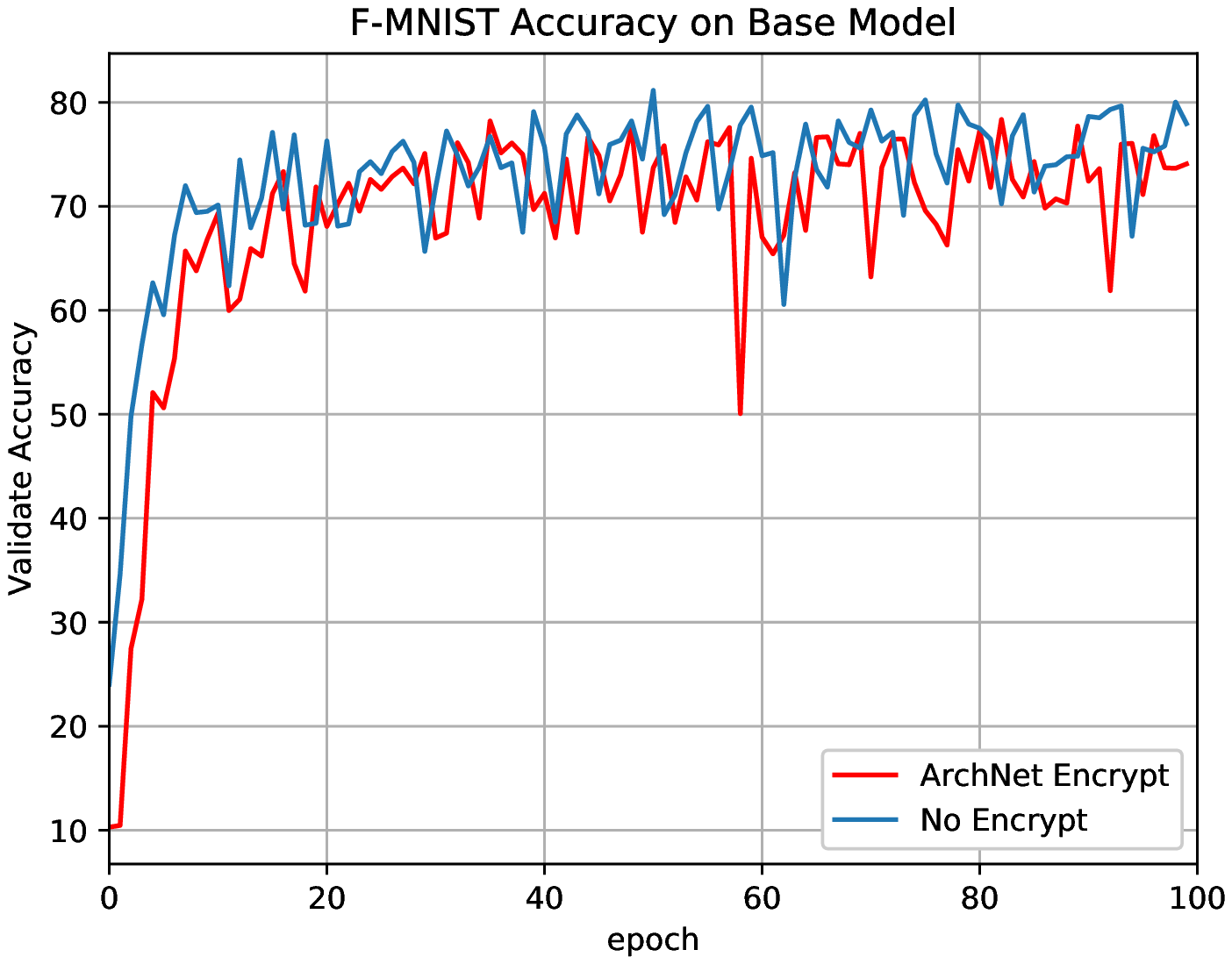}
    %\caption{fig1}
  }
  \quad
  \subfigure[pic2.]{
    \includegraphics[width=5.5cm]{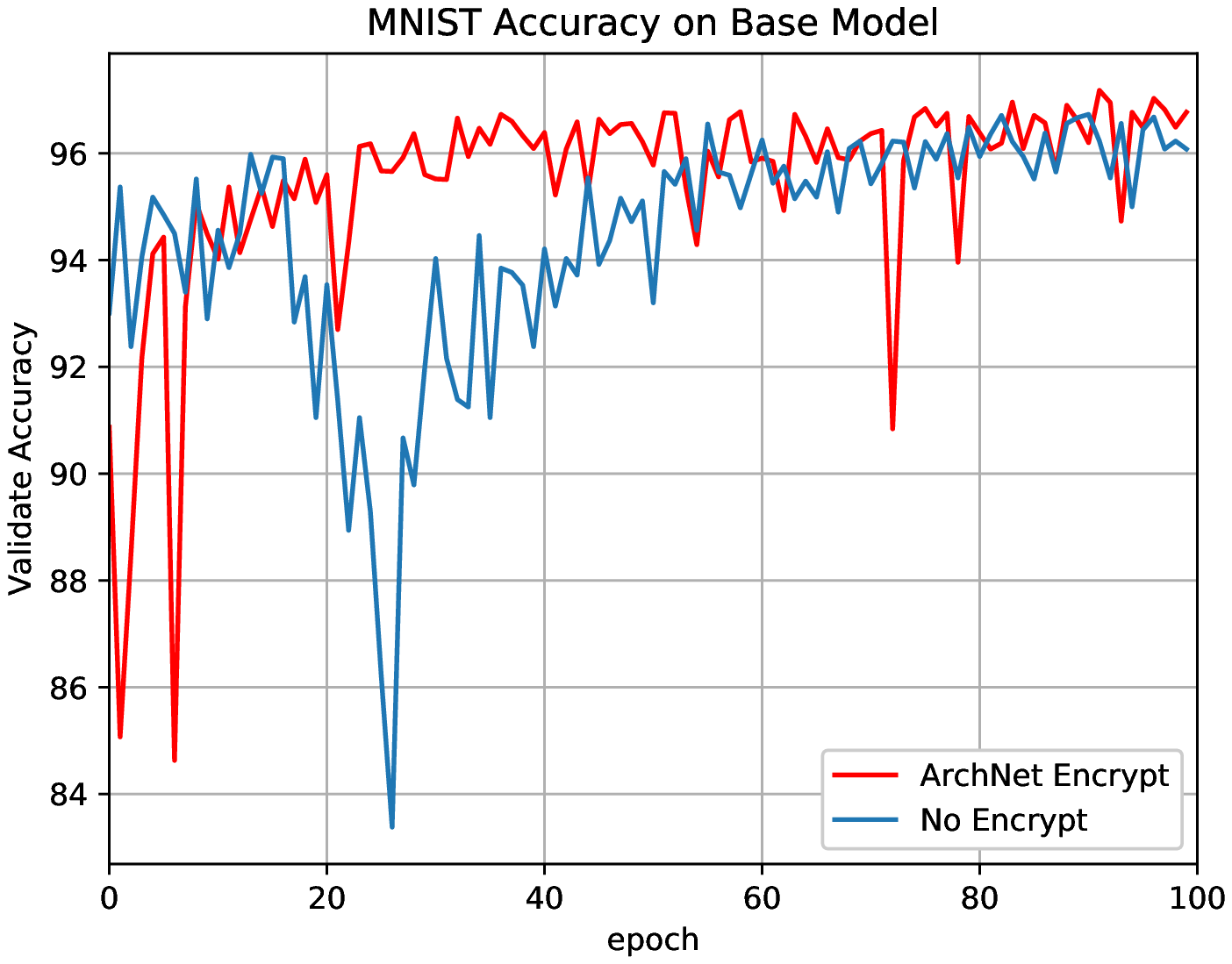}
  }
  \quad
  \subfigure[pic3.]{
    \includegraphics[width=5.5cm]{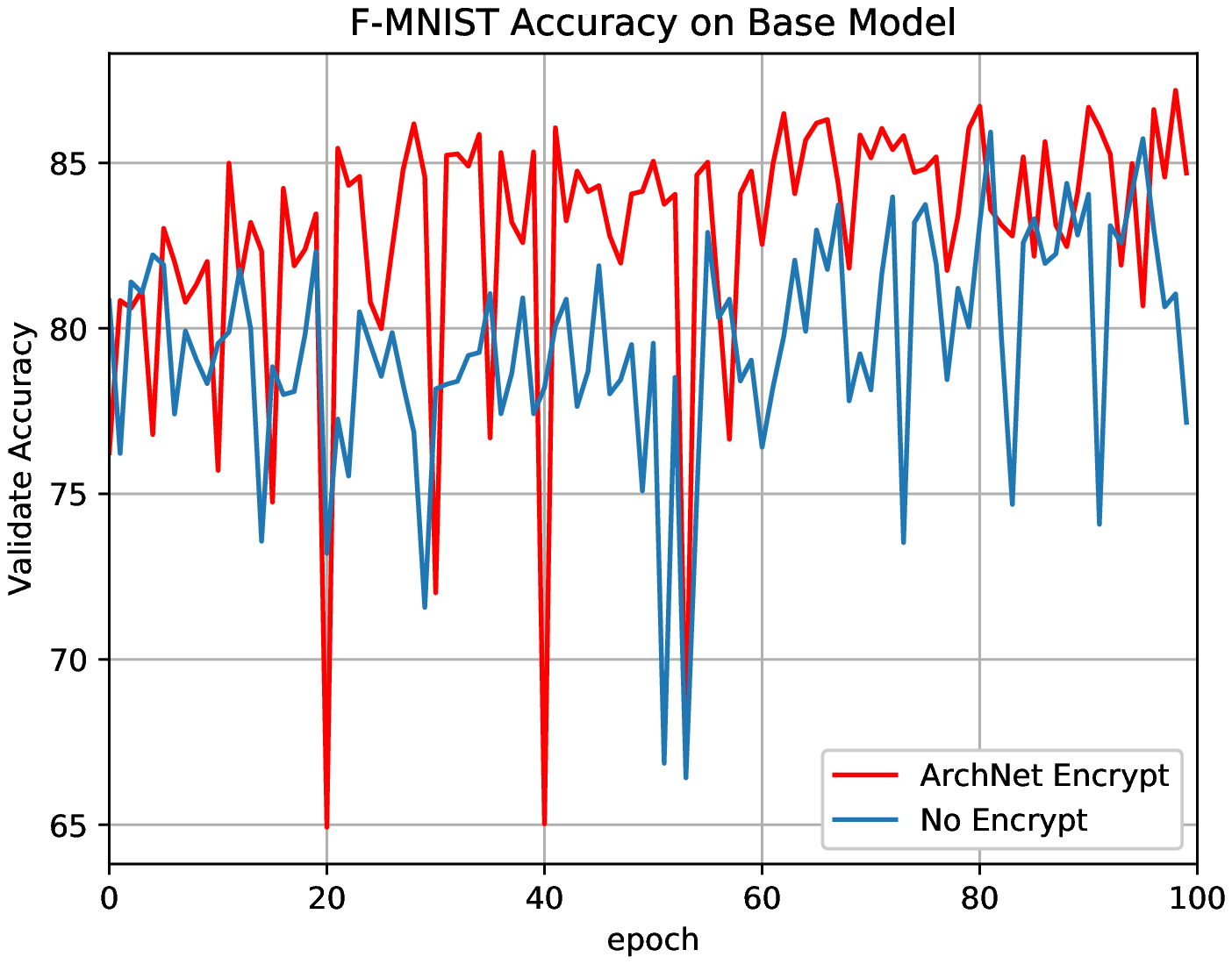}
  }
  \caption{The accuracy of the base model on different datasets. Subfigure (a) shows the training results on Fashion-MNIST dataset. Subfigure (b) shows the training results on MNIST dataset. Subfigure (c) shows the training results on Cifar-10 dataset.}
  \label{accuracy1}
\end{figure*}
% \usepackage{booktabs}
\begin{table*}[]
  \caption{ArchNet structure on different Datasets}
  \begin{center}
    \begin{tabular}{@{}cccc@{}}
      \toprule
      Dataset                    & \begin{tabular}[c]{@{}c@{}}MNIST\\ 1x28x28\end{tabular} & \begin{tabular}[c]{@{}c@{}}F-MNIST\\ 1x28x28\end{tabular} & \begin{tabular}[c]{@{}c@{}}Cifar-10\\ 3x32x32\end{tabular} \\ \midrule

      \begin{tabular}[c]{@{}c@{}}ArchNet\\ (H-encoder)\end{tabular} & \begin{tabular}[c]{@{}c@{}}Conv2d(1,3)\\ Conv2d(3,10)\\ Conv2d(10,10)\\ Linear(10*28*28,20*28*28)\\ Relu()\\ Conv2d(20,20)\\ ConvTrans2d(20,10)\end{tabular} & \begin{tabular}[c]{@{}c@{}}Conv2d(1,3)\\ Conv2d(3,10)\\ Conv2d(10,10)\\ Linear(10*28*28,20*28*28)\\ Relu()\\ Conv2d(20,20)\\ ConvTrans2d(20,10)\end{tabular} & \begin{tabular}[c]{@{}c@{}}Conv2d(3,3)\\ Conv2d(3,10)\\ Conv2d(10,20)\\ Conv2d(20,20)\\ ConvTrans2d(20,10)\end{tabular} \\ \midrule

      \begin{tabular}[c]{@{}c@{}}ArchNet\\ (L-decoder)\end{tabular} & \begin{tabular}[c]{@{}c@{}}Conv2d(10,10)\\ Conv2d(10,30)\\ Conv2d(30,10)\\ Conv2d(10,10)\\ Conv2d(10,10)\\ Relu()\\ Conv2d(10,5)\\ Conv2d(5,3)\\ Relu()\\ Conv2d(3,1)\end{tabular} & \begin{tabular}[c]{@{}c@{}}Conv2d(10,10)\\ Conv2d(10,30)\\ Conv2d(30,10)\\ Conv2d(10,10)\\ Conv2d(10,10)\\ Relu()\\ Conv2d(10,5)\\ Conv2d(5,3)\\ Relu()\\ Conv2d(3,1)\end{tabular} & \begin{tabular}[c]{@{}c@{}}Conv2d(10,10)\\ Conv2d(10,30)\\ Conv2d(30,10)\\ Conv2d(10,10)\\ Conv2d(10,10)\\ Relu()\\ Conv2d(10,5)\\ Conv2d(5,3)\\ Relu()\\ Conv2d(3,3)\end{tabular} \\ \midrule

      parameter                  & 122,960,821                & 122,960,821                & 14,961                     \\ \bottomrule
    \end{tabular}
  \end{center}
\end{table*}

\subsection{Analysis on the Difficulty of Stealing}
Pure convolution ArchNet maps the data A to a high dimensional form 
B. We select three dimensions from B to visualize. Subfigure a in 
Fig. \ref{pureconv} is an image generated by passing the MNIST dataset 
through ArchNet with pure convolution layers. Under the regular operation 
of convolution layer, human can easily distinguish the characteristics of dataset. 
The reason is the receptive field of convolution is similar to that of 
human eyes. Convolution layer is only a two-dimensional linear processing of data, 
and does not break up the original distribution of data. Fig. \ref{pureconv} 
uses ArchNet with activation function and fully-connected layer for simple dataset. 
It can break up the original distribution that the human can perceive. We can not 
obtain any other useful information through the visualization of encrypted data. 
Therefore, it is difficult to steal the original dataset through the ArchNet encrypted 
dataset as shown in Fig. \ref{purecif}.
\begin{figure}[htbp]
  \centering
  \subfigure[pic1.]{
    \includegraphics[width=1.5cm]{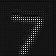}
    %\caption{fig1}
  }
  \quad
  \subfigure[pic2.]{
    \includegraphics[width=1.5cm]{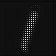}
  }
  \quad
  \subfigure[pic3.]{
    \includegraphics[width=1.5cm]{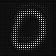}
  }
  \quad
  \subfigure[pic4.]{
    \includegraphics[width=1.5cm]{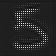}
  }
  \quad
  \subfigure[pic5.]{
    \includegraphics[width=1.5cm]{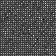}
    %\caption{fig1}
  }
  \quad
  \subfigure[pic6.]{
    \includegraphics[width=1.5cm]{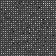}
  }
  \quad
  \subfigure[pic7.]{
    \includegraphics[width=1.5cm]{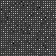}
  }
  \quad
  \subfigure[pic8.]{
    \includegraphics[width=1.5cm]{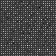}
  }
  \caption{Encryption on MNIST}
  \label{pureconv}
\end{figure}
\begin{figure}[htbp]
  \centering
  \subfigure[pic1.]{
    \includegraphics[width=1.5cm]{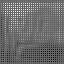}
    %\caption{fig1}
  }
  \quad
  \subfigure[pic2.]{
    \includegraphics[width=1.5cm]{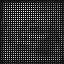}
  }
  \quad
  \subfigure[pic3.]{
    \includegraphics[width=1.5cm]{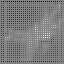}
  }
  \quad
  \subfigure[pic4.]{
    \includegraphics[width=1.5cm]{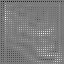}
  }
  \caption{Encryption on Cifar-10}
  \label{purecif}
\end{figure}
\begin{figure}[htbp]
  \centering
  \includegraphics[width=\linewidth]{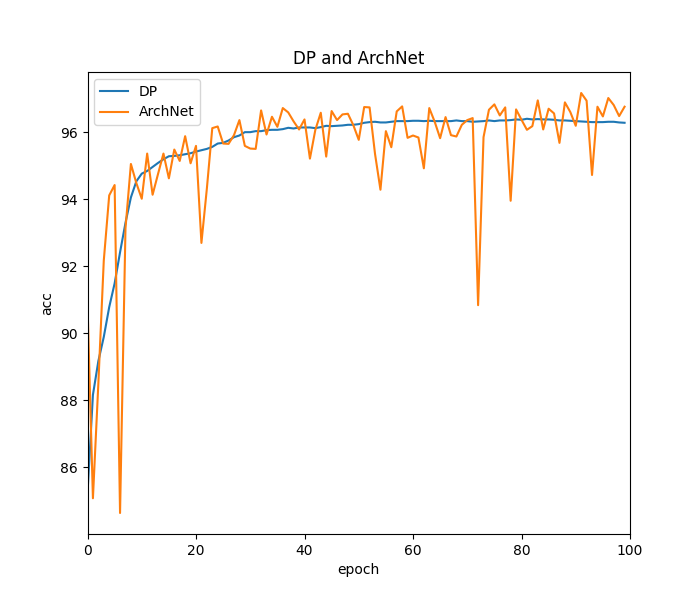}
  \caption{Accuracy compare between DP and ArchNet}
  \label{picl}
\end{figure}
\subsection{Operability Analysis}
The $EC$ value reflects the operability of the general encryption algorithm. In our experiment, the $EC$ value reflects the difference between the pattern recognition of the encrypted dataset and the pattern recognition of the original dataset in the case of this encryption method. We show the difference as follows, when the original dataset is MNIST, the dataset $EC$ value encrypted by ArchNet is 0.05\%, which is much smaller than the $EC$ value of 87.00\% encrypted by RC4 algorithm. Similar effects exist in different datasets. Our method is much better than general encryption algorithm in operability.

\subsection{Convergence Relation Analysis}

The convergence curve of the base model in the training process based on encrypted dataset is basically the same as that in the training process based on original dataset. It shows that the encrypted dataset of ArchNet is close to the original dataset in convergence relation. In the distributed machine learning system, this convergence relation proves that in the relation between $t_3$ and $t_1$ is very small, which is defined in section 4.7. Because the training curve of encrypted dataset is similar to that of original dataset, the computing end with computing resources can be optimized on the basis of the existing model without considering encryption methods. As shown in Fig. \ref{picl}, compared to the Difference Privacy policy with SGD, ArchNet is lack of stability but shows a small superiority on the accuracy.

\subsection{General Analysis}
The $EC$ values of the three datasets encrypted by ArchNet are less than 1\%. The $EC$ 
values of the three datasets encrypted by RC4 are around 87.00\%. The $EC$ value appears 
negative when validating the F-MNIST dataset, which make the samples enhanced after ArchNet 
maps the data to the high-dimensional space. The same base model is easier to classify the 
data in the high-dimensional space. If the same encryption method is applied to different 
datasets and the $EC$ value is similar, then the encryption method is independent of datasets. 
This encryption method has good universality. For the above ArchNet and RC4 algorithms, 
their $EC$ values are similar in different datasets. Therefore, the generality ArchNet is the same as RC4.

\section{RELATED WORKS}
% In this section, we demonstrate three areas that are related to the ArchNet.
\subsection{Distributed Machine Learning}
Distributed machine learning is a wide range of concepts, including multiple computing 
units of distributed learning model, big data distributed learning model, etc. Recently, 
there are researches on using distributed technology to improve the performance of traditional 
machine learning. Based on the concept of model sharing, a big data analysis system is 
introduced by Jie Jiang et al. \cite{JiangAngel} The high-dimensional big model is reasonably 
divided into multiple sub-model server nodes. Some researchers also apply the concept of 
distributed to specific scenarios, such as medical, legal and other fields \cite{court, HuangPatient, Hongmei}. 
However, the existing research of distributed machine learning system mainly focuses 
on the synchronization and data distribution of distributed machine learning \cite{Ho2013More,robust}. 
On the contrary, our distributed machine learning system focuses more on the innovation of business 
model, and uses ArchNet to solve the problem of data hiding in this business system.
\subsection{Neural Network Encryption}
The problem of neural network encryption is a hot topic. The theory of fully homomorphic encryption proposed by Gentry lays a foundation for encryption theory of complicated data \cite{Gentry2009A}. The CryptoNets neural network model proposed by Dowlin et al \cite{Xie2014Crypto}. It use FHE to realize deep learning of privacy protection. They provide a framework for designing neural networks that can run on encrypted data, and propose a polynomial approximation using the Relu activation function. CryptoNets and its derived neural network for solving encrypted data are gradually improved \cite{graph,JuvekarGazelle,ChouFaster}. Some researchers use encrypted neural network to solve the edge computing problem on IoT devices \cite{TianLEP}. However, the existing solution of neural network intends to design neural network with encrypted data. On the contrary, we research on how to encrypt data by certain methods. Therefore, the neural network without special processing can identify its pattern. It can increase the diversity of algorithm in the data computing end of distributed machine learning system.
\subsection{Model Stealing and Prevention}
In the application of machine learning with network, we need to solve the problem of model stealing. Model stealing refers to how to prevent data leakage when sensitive training data in machine learning model may leak personal privacy. Nicolas Papernot et al. proposes PATE that can differentiate the privacy data into different models through the strategies of students' and teachers' models to prevent model stealing \cite{PapernotScalable}. Yunhui Long et al. improve PATE method by GAN \cite{YunhuiScalable}. Some researches also put forward aggressive schemes from statistical machine learning model to deep learning model stealing and corresponding countermeasures \cite{KesarwaniModel,JuutiPRADA,apidl,article}. However, the existing method is to prevent the server from stealing the client's data in the machine learning of cloud computing, while our system is a machine learning system adopted in the distributed scenario, using two types of keys to lock the dataset. Our method makes the second type of key more flexible.
\section{CONCLUSIONS}
We considered the basic form of distributed machine learning system with embedded devices to address the limited hardware on cloud server. 
We made efforts to design a data hiding framework, ArchNet, to deal with the data stealing problem in distributed machine learning systems. 
Our experiment can well prove the correctness of the data hiding principle proposed in this paper. Compared 
with the traditional encryption algorithm, the model based on neural network can well complete the encryption and decryption 
tasks in the distributed machine learning system from the aspects of difficulty to steal and operability. ArchNet is useful in the emerging machine learning applications, such as 3D picture 
transmission and remote model extraction on embedded system.

\bibliographystyle{plainnat}
\bibliography{thispaper}

\end{document}